\documentclass[11pt,letterpaper]{article}
\usepackage[margin=1in]{geometry}
\IfFileExists{arxiv_submission/math_commands.tex}{\def\arxivpath{arxiv_submission/}}{\def\arxivpath{}}
\makeatletter
\edef\input@path{{\arxivpath}}
\makeatother
\PassOptionsToPackage{dvipsnames}{xcolor}

\usepackage{amsmath,amsfonts,bm}

\def\1{\bm{1}}

\DeclareMathAlphabet{\mathsfit}{\encodingdefault}{\sfdefault}{m}{sl}
\SetMathAlphabet{\mathsfit}{bold}{\encodingdefault}{\sfdefault}{bx}{n}

\newcommand{\E}{\mathbb{E}}

\newcommand{\R}{\mathbb{R}}

\newcommand{\dd}{\, {\rm d}}
\newcommand{\defeq}{\coloneqq}
\newcommand{\cut}[1]{}

\newcommand{\namelong}{\textsc{Kinematic Interpolants}\xspace}
\newcommand{\nameshort}{\textsc{KI}\xspace}

\usepackage[framemethod=TikZ]{mdframed}
\mdfdefinestyle{MyFrame}{linecolor=black,
    outerlinewidth=.3pt,
    roundcorner=5pt,
    innertopmargin=1pt, innerbottommargin=1pt, innerrightmargin=1pt,
    innerleftmargin=1pt,
    backgroundcolor=black!0!white}

\mdfdefinestyle{MyFrame2}{linecolor=white,
    outerlinewidth=1pt,
    roundcorner=1pt,
    innertopmargin=5pt,
    innerbottommargin=5pt,
    innerrightmargin=10pt,
    innerleftmargin=10pt,
    backgroundcolor=black!3!white}

\mdfdefinestyle{MyFrameEq}{linecolor=white,
    outerlinewidth=0pt,
    roundcorner=0pt,
    innertopmargin=0pt,
    innerbottommargin=0pt,
    innerrightmargin=7pt,
    innerleftmargin=7pt,
    backgroundcolor=black!3!white}

\mdfdefinestyle{nfmdefinitionstyle}{
    backgroundcolor=black!3,
    linecolor=black!35,
    linewidth=0.5pt,
    roundcorner=0pt,
    innerleftmargin=9pt,
    innerrightmargin=9pt,
    innertopmargin=6pt,
    innerbottommargin=6pt,
    skipabove=10pt,
    skipbelow=10pt
}

\usepackage{graphicx}
\graphicspath{{\arxivpath}}

\usepackage{hyperref}
\usepackage{url}

\usepackage[utf8]{inputenc} \usepackage[T1]{fontenc}    \usepackage{booktabs}       \usepackage{threeparttable}
\usepackage{amsfonts}       \usepackage{nicefrac}       \usepackage{microtype}      \usepackage[dvipsnames]{xcolor}

\usepackage{enumitem}
\usepackage{mathtools}
\usepackage{todonotes}
\usepackage{amssymb,fge}
\usepackage{wrapfig}
\usepackage{bbm}
\usepackage{mathrsfs}
\usepackage{soul}
\usepackage{array}
\usepackage{multirow}
\usepackage{tikz}
\usepackage{subcaption}
\usetikzlibrary{positioning,arrows.meta,shapes.geometric}
\usepackage[round]{natbib}
\usepackage{listings}
\usepackage{xcolor}         \usepackage{colortbl}
\usepackage{xspace}
\usepackage{lineno}
\usepackage{titletoc}

\usepackage{amsmath, amsthm, mathtools}
\usepackage{cleveref}
\usepackage[normalem]{ulem}
\usepackage{changepage}
\usepackage{float}
\usepackage{needspace} \usepackage{makecell}
\usepackage{pifont}

  \definecolor{sectiongray}{gray}{0.90}
\definecolor{refgray}{gray}{0.55}
\usepackage[ruled,vlined,linesnumbered]{algorithm2e}

\newtheorem{theorem}{Theorem}
\newtheorem{lemma}[theorem]{Lemma}

\newtheorem{proposition}[theorem]{Proposition}

\theoremstyle{definition}
\newmdtheoremenv[
    style=nfmdefinitionstyle
]{nfmdefinition}{Definition}
\theoremstyle{plain}

\newtheoremstyle{nfmremarktext}
    {5pt}{0pt}
    {\normalfont}{}
    {\bfseries}{.}{0.5em}{}

\theoremstyle{nfmremarktext}
\newmdtheoremenv[
    backgroundcolor=white,
    linecolor=black!25,
    linewidth=0.4pt,
    roundcorner=0pt,
    topline=true,
    bottomline=true,
    leftline=true,
    rightline=true,
    innerleftmargin=9pt,
    innerrightmargin=9pt,
    innertopmargin=6pt,
    innerbottommargin=6pt,
    skipabove=8pt,
    skipbelow=8pt
]{nfmremark}{Remark}
\theoremstyle{plain}

\hypersetup{
    pdftitle={Learning smooth dynamics via Kinematic Interpolants},
    pdfauthor={Yiheng Roger Jiang; Youssef Marzouk; Joey Bose; Aram-Alexandre Pooladian},
	colorlinks=true,       linkcolor=RoyalBlue,        citecolor=RoyalBlue,        filecolor=magenta,     urlcolor=RoyalBlue
}

\newcommand{\vscore}[2]{\ensuremath{#1\!\pm\!#2}}
\newcommand{\vbest}[2]{\ensuremath{\boldsymbol{#1\!\pm\!#2}}}

\newcommand{\OurTtwo}{\vscore{0.0691}{0.0039}}
\newcommand{\OurTfour}{\vscore{0.0224}{0.0006}}
\newcommand{\OurTsix}{\vscore{0.0278}{0.0017}}
\newcommand{\OurTeight}{\vscore{0.0639}{0.0019}}
\newcommand{\OurFromFirst}{\vbest{0.1552}{0.0167}}
\newcommand{\OurFromLast}{\vbest{0.0281}{0.0066}}
\newcommand{\OurIndependentTtwo}{\vscore{0.0656}{0.0040}}
\newcommand{\OurIndependentTfour}{\vbest{0.0217}{0.0008}}
\newcommand{\OurIndependentTsix}{\vbest{0.0266}{0.0029}}
\newcommand{\OurIndependentTeight}{\vscore{0.0635}{0.0035}}
\newcommand{\OurIndependentFromFirst}{\vscore{0.1685}{0.0462}}
\newcommand{\OurIndependentFromLast}{\vscore{0.0720}{0.0132}}

\begin{document}
\title{Learning smooth dynamics via Kinematic Interpolants}
\author{{Yiheng Jiang}\thanks{Massachusetts Institute of Technology.
\texttt{rogeryj@mit.edu}, \texttt{ymarz@mit.edu}}\qquad
{Youssef Marzouk}\footnotemark[1]\\[0.5em]
{Avishek Joey Bose}\thanks{Imperial College London and Mila -- Qu\'ebec AI Institute.
\texttt{joey.bose@imperial.ac.uk}}\qquad
{Aram-Alexandre Pooladian}\thanks{Institute for Foundations of Data Science, Yale University.
\texttt{aram-alexandre.pooladian@yale.edu}}}
\date{}
\maketitle
\begin{abstract}
\looseness=-1
Inferring dynamics from real-world experimental data is inherently ambiguous, as many plausible dynamical systems can explain the same positional observations at measurement times. For higher fidelity modeling of applications ranging from single-cell RNA to vortex dynamics and even flocking of birds, scientific observables increasingly also include \emph{velocity} estimates that serve to resolve this issue. However, incorporating this rich information in a scalable manner, free of simulation or unnecessary approximations, has remained a difficult challenge. In this paper, we introduce \namelong (\nameshort), a novel framework that extends the seminal methodology of stochastic interpolants to infer nonlinear dynamics across multiple snapshots in phase space, i.e., observables with position and velocity measurements. Empirically, we demonstrate that across several scientific datasets, \nameshort achieves superior reconstruction of the underlying dynamics compared to existing baselines while requiring only a fraction of the training time.

\end{abstract}
\section{Introduction}

\looseness=-1
In many scientific applications, we want to model a complex underlying process that produces potentially high-throughput measurements over time. Such experimental observables often capture evolving population dynamics through multiple irregular snapshots in time; and thus, observations do not follow individual members of the governing scientific process. For instance, single-cell RNA (scRNA) sequencing is inherently \emph{destructive}, where the measurement process destroys the cells; this implies that the same cell cannot be tracked from one time snapshot to the next. As a result, experimental observables are curated in the form of unpaired samples from the population at different times. The underlying task is then to faithfully recover the underlying continuous dynamics that are consistent with the available population measurements, commonly known as the \emph{trajectory inference} problem~\citep{schiebinger2019optimal,tong2020trajectorynet,lavenant2024toward}.

\looseness=-1
The very nature of population measurements means that simply observing positions of particles is not sufficient to fully disambiguate the governing process. Rotationally symmetric dynamics around a fixed center are a canonical example of this phenomenon, which can yield identical snapshot measurements from the perspective of positions alone. Increasingly, experimental observations now also record the infinitesimal evolution of populations, in conjunction with standard measurements of positions. Notable examples include RNA velocity~\citep{lamanno2018rna}, which estimates the direction in which each cell's gene expression is changing; drifting buoys in open water measuring surface velocities of ocean currents~\citep{elipot2016global}; and three-dimensional tracking
of bird flocks that gives the velocity of each tracked bird~\citep{ling2019costs,cavagna2022marginal}. These additional velocity measurements help to remove much of the ambiguity of the original trajectory inference problem by providing additional constraints on the reconstruction. However, this rich signal is only exploitable for probabilistic models that can represent it.

\looseness=-1
Existing approaches seek to incorporate velocity information through different modeling choices. To learn non-gradient field dynamics, Curly-FM~\citep{petrovic2025curlyflowmatchinglearning} proposes to first construct a reference drift from approximate velocity estimates, learn an interpolant that follows it, and then fit a (first-order) flow. Wasserstein Lagrangian Mechanics (WLM)~\citep{guan2026lagrangianactionlearningpopulation} learns a population-level potential governing second-order population dynamics, but training requires expensive trajectory-based simulation. More recently, Acceleration Matching~\citep{dazzini2026trajectory} and TracingFlow~\citep{sun2026tracingflowsimulationfreetrajectoryinference} learn acceleration fields in a \emph{simulation-free} manner, but are restricted to first algorithmically imputing the missing velocity data for each observation before training.  We argue that readily available position-velocity data calls for \emph{probabilistic modeling} of second-order dynamics that can still benefit from simulation-free training algorithms.

\paragraph{Contributions.}
In this work, we introduce \namelong (\nameshort), an extension of stochastic interpolants~\citep{lipman2023flowmatchinggenerativemodeling,albergo2025stochasticinterpolantsunifyingframework} for learning dynamics from multiple snapshots of position-velocity data, where the learned field sets the acceleration across time. \nameshort designs an interpolant on phase space such that the velocity is defined to be the derivative of the position state, ensuring that both position and velocity are consistent by construction.\looseness=-1

Our main contributions are as follows:

\begin{enumerate}[topsep=0pt, partopsep=0pt, itemsep=0pt, parsep=0pt, leftmargin=*]
\item \looseness=-1 We show that conditioning the pathwise acceleration on the current state gives a field whose dynamics preserve the interpolant's one-time marginals under the assumptions of Proposition~\ref{prop:kinetic_transport} (Section~\ref{sec:steering_dynamics}).
\item \looseness=-1 We obtain closed-form acceleration targets by differentiating the interpolant's velocity component, enabling scalable, simulation-free regression training of an acceleration network without building a reference field or first learning an interpolant (Section~\ref{sec:sim_free_training}).
\item \looseness=-1 We develop a flexible family of interpolants with different noise schedules and support both deterministic and stochastic sampling. For stochastic sampling, we learn the velocity score to obtain the drift correction required for adjustable diffusion schedules (Sections~\ref{sec:path_choices} and~\ref{sec:learning_stochastic}).
\item \looseness=-1 Across several scientific datasets, we demonstrate that \nameshort improves reconstruction accuracy at substantially lower training cost than existing baselines (Section~\ref{sec:experiments}).
\end{enumerate}

\section{Background}\label{sec:phase_space_background}
Throughout, we assume all quantities are sufficiently smooth for the calculations below, and all displayed expectations and integrals are finite. We write ${\cal P}(\R^d)$ to be the space of {probability densities} over $\R^d$. For $P_1,\ldots,P_M \in {\cal P}(\R^d)$, we write $\Gamma(P_1,\ldots,P_M)$ to be the set of (multimarginal) couplings between the probability measures. For $\pi \in \Gamma(P_1,\ldots,P_M)$, we say $\pi_{\ell,k} \in \Gamma(P_\ell,P_k)$ for $\ell,k\in[M]$. For a time-dependent function $f(t)$, we write $\partial_t f(t) \equiv \dot{f}(t)$. To alleviate notational burden, we will often write $(x,v)$ to mean $(x_t,v_t)$ when convenient.

\looseness=-1
\paragraph{Flows on phase space.}
Phase space is defined as the joint space of positions and velocities, which in our case is simply denoted by $\R^d\times \R^d \cong \R^{2d}$. If we now want to evolve $Z_0 = (X_0,V_0) \sim \rho_0 \in {\cal P}(\R^{2d})$, then the flow must be given by $b_t : \R^{2d} \to \R^{2d}$ with $b_t=(b_t^x,b_t^v)$ and
\begin{equation*}
\begin{aligned}
\dot X_t=b_t^x(X_t,V_t)\,,\quad \dot V_t=b_t^v(X_t,V_t)\,,
\end{aligned}
\end{equation*}
where the density $\rho_t\defeq {\rm Law}((X_t,V_t))$ evolves according to the continuity equation on the joint space
\begin{equation*}
\partial_t\rho_t + \nabla\cdot(b_t\rho_t) =
\partial_t\rho_t
+\nabla_x\cdot(b_t^x \rho_t)
+\nabla_v\cdot(b_t^v \rho_t)=0\,.
\end{equation*}

To have the dynamics resemble a physical system (otherwise, a priori, the velocity and position need not interact), we restrict $b_t^x(x,v)=v$, and, writing $b_t^v=a_t$, we obtain
\begin{equation*}
\dot X_t=V_t\,,\qquad
\dot V_t=a_t(X_t,V_t)\,,
\end{equation*}
or, equivalently,
$\ddot X_t=a_t(X_t,\dot X_t)$, and consequently, the \emph{kinetic} continuity equation
\begin{equation}
\label{eq:background_kinetic_continuity}
\partial_t \rho_t
+\nabla_x\cdot(v\rho_t)
+\nabla_v\cdot(a_t \rho_t)=0\,.
\end{equation}

\looseness=-1
\paragraph{Diffusions on phase space.}
\label{sec:background_diffusions_scores}
To have stochastic sample paths on phase space, it is standard to consider the following dynamics
\begin{equation}
\label{eq:background_kinetic_diffusion}
\dd X_t=V_t\dd t\,,\quad \dd V_t=c_t(X_t,V_t)\dd t+\sqrt{\beta_t}\dd B_t\,,
\end{equation}
where $B_t$ is a Brownian motion on $\R^d$, $\beta_t$ is a non-negative diffusion parameter, and we replace the acceleration field by $c_t$. The law of the joint variables then follows the \emph{kinetic} Fokker--Planck equation\looseness-1
\begin{equation}
\label{eq:background_kinetic_fokker_planck}
\begin{aligned}
\partial_t\rho_t+\nabla_x\cdot(v\rho_t)+\nabla_v\cdot(c_t\rho_t)=\frac{\beta_t}{2}\Delta_v\rho_t\,,
\end{aligned}
\end{equation}
where $\Delta_v$ is the Laplacian in the velocity variable. To have the $t$-marginals of \eqref{eq:background_kinetic_fokker_planck} be the same as \eqref{eq:background_kinetic_continuity}, choose
\begin{align}\label{eq:background_score_corrected_drift}
\begin{split}
c_t(x,v)=a_t(x,v)+\frac{\beta_t}{2}s_t(x,v)  \defeq a_t(x,v)+\frac{\beta_t}{2}\nabla_v \log \rho_t(x,v)\,,
\end{split}
\end{align}
where the second term above, the \emph{velocity-score}, cancels with the diffusion term.

\section{Kinematic interpolants}
\label{sec:phase_space_si}
\looseness=-1
{We consider observations collected at irregular times $0=t_0 < t_1 < \cdots < t_J=1$. At each observation time $t_i$, we have samples from a joint position--velocity distribution:
\begin{align*}
Z_{t_i} = \bigl(X_{t_i},V_{t_i}\bigr) \sim \mu_{t_i} \in {\cal P}(\R^{2d})\,.
\end{align*}
While a single datum pairs positions and velocities {within} a snapshot, the identities of the individual particles are unavailable across snapshots. We therefore assume the underlying population consists of independent particles with continuous trajectories, and our goal is to learn a process $Z_t^\theta=(X_t^\theta,V_t^\theta)$ with laws
\begin{align*}
  \mu^{\theta}_t = \operatorname{Law}\bigl(X^{\theta}_t,V^{\theta}_t\bigr)\,, \quad t\in[0,1]\,.
\end{align*}
Our approach will be to minimize the discrepancy between $\mu_{t_i}$ and $\mu^\theta_{t_i}$ at all observation times. The learned process then generates joint position--velocity samples at intermediate times. To construct the training targets, we introduce an interpolant with law $\rho_t=\operatorname{Law}(X_t,V_t)$ satisfying $\rho_{t_i}=\mu_{t_i}$ at each observation time.}

\subsection{Kinematically consistent paths}
\label{sec:consistent_paths}
Let $\pi \in \Gamma(\mu_{t_0}, \ldots, \mu_{t_J})$ be a fixed coupling of the phase space marginals. We now choose functions
\begin{align*}
    I_t:\mathbb R^{2d}\times\mathbb R^{2d}
    \to\mathbb R^d\,,\quad G_t\in \mathbb R^{d\times 2d}
\end{align*}
and construct phase space interpolants as follows: for an arbitrary index $j \in \{0,\ldots,J-1\}$ and $t\in[t_j,t_{j+1}]$, draw $(Z_{t_j},Z_{t_{j+1}}) \sim \pi_{j,j+1}$ and define
\begin{equation}
\label{eq:phase_si_position}
    X_t
    =
    I_t(Z_{t_j},Z_{t_{j+1}})
    +
    G_t\xi_j\,,
\end{equation}
where $\xi_j\sim\mathcal N(0,I_{2d})$. Conceptually, $I_t$ defines the interpolation path constrained by the endpoints $Z_{t_j}$ and $Z_{t_{j+1}}$, and the second term adds noise along the path. The velocity of said path is then
\begin{align}
    V_t&:=\dot X_t
    =\partial_t I_t(Z_{t_j},Z_{t_{j+1}})+\dot G_t\xi_j\,,\label{eq:phase_si_velocity}
\end{align}
which results in the random variable for $t \in (t_{j},t_{j+1})$
\begin{align}\label{eq:zt_main}
Z_t &= \begin{pmatrix}
    X_t \\
    V_t
\end{pmatrix}
=
\begin{pmatrix}
    {I_t(Z_{t_j},Z_{t_{j+1}})} + G_t \xi_j\\
    {\partial_t I_t(Z_{t_j},Z_{t_{j+1}})} + \dot G_t \xi_j
\end{pmatrix}\,.
\end{align}

Conditioned on the endpoints (which are available to us), $Z_t$ is simply a Gaussian random variable in $\R^{2d}$ with mean and covariance given by
\begin{align}
m_t &\defeq\begin{pmatrix}
    I_t(Z_{t_j},Z_{t_{j+1}}) \\
    \partial_t I_t(Z_{t_j}, Z_{t_{j+1}})
\end{pmatrix}\,,\label{eq:mean_eq}\\
    \Sigma_t &\defeq \begin{pmatrix}
G_tG_t^\top&G_t\dot G_t^\top\\
\dot G_tG_t^\top&\dot G_t\dot G_t^\top
\end{pmatrix}\,.\label{eq:matrix_noise_covariance}
\end{align}

Our choice of interpolant is then only valid (meaning, satisfies the endpoint constraints) under the following suitable restrictions for the mean
{
\begin{align}
    I_{t_j} = X_{t_j}\,,\,\,\,& I_{t_{j+1}} = X_{t_{j+1}}\,,\label{eq:mean_boundary_conditions_1}\\
    \partial_t I_{t_j} = V_{t_j}\,,\,\,\,& \partial_t I_{t_{j+1}}= V_{t_{j+1}}\label{eq:mean_boundary_conditions_2}\,,
\end{align}
and the following for the covariance:
\begin{equation}
\label{eq:noise_boundary_conditions}
    (G_{t_j},\dot G_{t_j}) = (G_{t_{j+1}},\dot G_{t_{j+1}}) = 0\,.
\end{equation}
}
\subsection{Learning dynamics via accelerations}
\label{sec:steering_dynamics}
The preceding section relies on the fact that our observations are in the form of position-velocity pairs, so we can impose the constraints in~\eqref{eq:mean_boundary_conditions_1} and~\eqref{eq:mean_boundary_conditions_2}. Now, to learn the relevant dynamics to match these constraints, we incorporate a final missing ingredient, namely (conditional) accelerations of the chosen path:\looseness-1
\begin{align}\label{eq:phase_si_acceleration}
        A_t^{(j)}&:=\dot V_t
    =\partial_{t}^2 I_t(Z_{t_j},Z_{t_{j+1}})+\ddot G_t\xi_j\,,
\end{align}
which follows by differentiating \eqref{eq:phase_si_velocity} in time. To obtain an acceleration field depending only on the current time, position, and velocity, we take the conditional mean:
\begin{equation}
\label{eq:markovian_acceleration}
\mathbb E\left[A_t^{(j)}\mid (X_t,V_t)=(x,v)\right]\,.
\end{equation}
Here the expectation averages over the sampled endpoints and noise conditional on the interpolant passing through $(x,v)$ at time $t$. Proposition~\ref{prop:kinetic_transport} below shows that the acceleration field given by  \eqref{eq:markovian_acceleration} reproduces the interpolant's joint position--velocity distribution at every time, in particular at the the prescribed marginals $\mu_{t_j}$.
\begin{mdframed}[style=MyFrame2]
\begin{proposition}
\label{prop:kinetic_transport}
Given any multimarginal coupling $\pi\in\Gamma(\mu_{t_0},\ldots,\mu_{t_J})$, construct $(X_t,V_t)$ as in Section~\ref{sec:consistent_paths}, and choose
\begin{align*}
    \!\!\!\!a_t^M(x,v) = \sum_{j=0}^{J-1}\E\left[A_t^{(j)}\mid (X_t,V_t)=(x,v)\right]\,.
\end{align*}
Then the dynamics
\begin{equation}
\label{eq:markovian_ode}
\dot X_t^M=V_t^M\,,\qquad
\dot V_t^M=a_t^M(X_t^M,V_t^M)\,,
\end{equation}
initialized with $(X_0^M,V_0^M)\sim\mu_{t_0}$, have the same marginals as the interpolant: for every $t\in[0,1]$,
\begin{equation}
\label{eq:markovian_marginal_matching}
\operatorname{Law}(X_t^M,V_t^M)=\rho_t\,.
\end{equation}
In particular, they pass through each prescribed marginal $\mu_{t_j}$ at time $t_j$.
\end{proposition}
\end{mdframed}
A short proof can be found in Appendix~\ref{app:kinetic_transport_pf}.

\subsection{On the choice of paths}
\label{sec:path_choices}
We first highlight that conditioned on the endpoints, our phase space interpolant $Z_t$ defined in \eqref{eq:zt_main} is distributed according to a Gaussian with mean $m_t$ via \eqref{eq:mean_eq} and covariance $\Sigma_t$ via \eqref{eq:matrix_noise_covariance}. We now describe some canonical choices for these objects. Throughout, we will consider an arbitrary interval $[t_j,t_{j+1}]$ and write everything in a reparameterization $u = (t-t_j)/\delta \in [0,1]$ where $\delta = t_{j+1} - t_j$. We also write $E_j = (Z_{t_j}, Z_{t_{j+1}})$.
\subsubsection{Choosing the mean path}
Recall that the position interpolant $I_t\in\R^d$ and its time derivative must satisfy the four constraints in~\eqref{eq:mean_boundary_conditions_1} and~\eqref{eq:mean_boundary_conditions_2}. A natural choice is the cubic Hermite polynomial, as for fixed endpoint positions and velocities, it is the unique polynomial of degree at most three satisfying all four boundary conditions:
\begin{equation}
\label{eq:hermite_mean}
\begin{aligned}
I_t(E_j) = H_t(E_j)&=(2u^3-3u^2+1)X_{t_j} +\delta(u^3-2u^2+u)V_{t_j} \\
&\phantom{=}+(-2u^3+3u^2)X_{t_{j+1}} +\delta(u^3-u^2)V_{t_{j+1}}\,.
\end{aligned}
\end{equation}
Its first and second time derivatives are available in closed form, giving the mean velocity and the deterministic component of the acceleration targets.

\subsubsection{Choosing the noise}
\label{sec:noise_choices}
The Gaussian covariance can also be significantly simplified by making particular choices. To start, we choose $G_t=\ell_t\otimes I_d$, $\ell_t\in\mathbb R^{1\times2}$
\begin{equation}
\label{eq:polar_l}
    \ell_t
    =
    \gamma_t
    \begin{pmatrix}
        \cos\theta_t & \sin\theta_t
    \end{pmatrix}\,,
\end{equation}
which implies that the covariance is
\begin{equation}
\label{eq:trig_covariance}
    \Sigma_t = \begin{pmatrix}
        \gamma_t^2
        &
        \gamma_t\dot\gamma_t
        \\
        \gamma_t\dot\gamma_t
        &
        \dot\gamma_t^2+\gamma_t^2\dot\theta_t^2
    \end{pmatrix} \otimes I_d\,.
\end{equation}
The parameter $\gamma_t$ controls position noise, while the changing angle supplies noise to the velocity component that is not determined by the positional noise alone. Taking $\theta_t = \omega t$ for $\omega > 0$, we can choose for instance
\begin{align*}
    \gamma_t \in \{u^2(1-u)^2\gamma\,, \gamma\}\,,
\end{align*}
for $\gamma > 0$. The first choice (bridge noise) satisfies the endpoint constraints, whereas constant noise instead convolves each endpoint law with a Gaussian. We emphasize that our approach is significantly more general than the choice of noise schedule in Acceleration Matching~\citep{dazzini2026trajectory}; see Appendix~\ref{sec:am_special_case} for more details on this point.

\section{Training and sampling}\label{sec:algorithms}
\subsection{Simulation-free training algorithm}\label{sec:sim_free_training}
In light of \eqref{eq:markovian_acceleration} being a valid acceleration field that generates the marginal data (recall Proposition~\ref{prop:kinetic_transport}), a natural algorithm is to regress a parametric neural network onto the accelerations $A_t^{(j)}$ from \eqref{eq:phase_si_acceleration}. Indeed, conditional expectation is the least-squares predictor of the pathwise acceleration from $(t,X_t,V_t)$, and so we can learn this field by regressing directly on sampled interpolant states and their accelerations. The resulting objective function of interest is then
\begin{equation}
\label{eq:acceleration_regression}
\!\!\!\!\!\!\mathcal L_{\mathrm{KI}}(\psi)=\sum_{j=0}^{J-1}\int_{t_j}^{t_{j+1}}\!\!\!\mathbb E\bigl[
 w(t)\|a_\psi(t,Z_t)-A_t^{(j)}\|^2\bigr] \dd t\,,
\end{equation}
where $w(\cdot)>0$ is a weight function, and we recall $A_t^{(j)}$ corresponds to the $j^{\rm th}$ interval. Notably, our algorithm is completely simulation-free.

Concretely, for a sampled interval $j \in \{0,1,\ldots,J-1\}$ and sampled time $t \sim {\rm Unif}(t_j,t_{j+1})$, we sample minibatches of endpoint pairs $(Z_{t_j},Z_{t_{j+1}}) \sim \pi_{j,j+1}$ from the prescribed coupling, where  $\pi_{j,j+1} \in \Gamma(\mu_{t_j},\mu_{t_{j+1}})$. We then compute $Z_t$ as in \eqref{eq:zt_main} following our choices for $I_t$ and $\Sigma_t$, as well as $A_t^{(j)}$, and perform gradient descent on the model parameters. See Algorithm~\ref{alg:kinematic_flow_matching} for a description.
\subsection{Learning stochastic dynamics}\label{sec:learning_stochastic}
\looseness=-1
Training with \eqref{eq:acceleration_regression} results in a sampling scheme that obeys the continuity equation. To introduce stochasticity in the sampling dynamics, we would instead have to obey \eqref{eq:background_kinetic_diffusion}, which requires learning the score function $s_t=\nabla_v\log\rho_t$. As with stochastic interpolants~\citep{albergo2025stochasticinterpolantsunifyingframework}, we also have a similar regression objective here.

Indeed, as the conditional law of $Z_t$ is given by
\begin{align*}
Z_t\mid E_j\sim\mathcal N\bigl(m_t(E_j),\Sigma_t\bigr)\,,\,\,
m_t(E_j)=\begin{pmatrix}I_t(E_j)\\\partial_tI_t(E_j)\end{pmatrix}\,,
\end{align*}
it follows that the conditional (velocity) score is
\begin{align*}
s_t^{v}(z;E_j)
&:=\nabla_v\log\rho_t(z\mid E_j) =-\begin{bmatrix}0&I_d\end{bmatrix}
\Sigma_t^{-1}\bigl(z-m_t(E_j)\bigr)\,.
\end{align*}
Since $\mathbb E[s_t^{v}(Z_t;E_j)\mid Z_t=z]
=\nabla_v\log\rho_t(z)$, a natural objective function emerges as before:
\begin{align*}
    \mathcal L_{\mathrm{SC}}(\phi)=\int_0^1 \mathbb E\!\left[
\lambda(t)\|s_\phi(t,Z_t)-s_t^{v}(Z_t;E_j)\|^2
\right]\dd t\,.
\end{align*}
\subsubsection{Minibatch training}
Per Algorithm~\ref{alg:kinematic_flow_matching}, after sampling a time $t$ and identifying an interval $j$, we draw samples from a coupling of $\mu_{t_j}$ and $\mu_{t_{j+1}}$. A first choice is the independent coupling, which is trivial to implement. A second is to consider a pairing induced by solving optimal transport problems on the minibatch, which has been shown to lower sampling costs in the two-marginal case~\citep{tong2024improvinggeneralizingflowbasedgenerative,pooladian2023multisample}. As our data takes the form of positions and velocities at the two end-points, we will always use the following cost function
\begin{align}\label{eq:ot_cost}
\begin{split}
c((x,v),(y,w)) =\frac{12}{\delta^3}\!\left\|y-x-\frac{\delta}{2}(v+w)\right\|^2\!\!\!+\! \frac{1}{\delta}\|w-v\|^2\,,
\end{split}
\end{align}
when computing the couplings, motivated by existing work on kinetic optimal transport~\citep{chen2018measure,
brigati2025kinetic}. We consider both pairwise OT couplings (using all available data prior to training), and minibatch OT couplings in our experiments.

\begingroup
\setlength{\algomargin}{2em}
\begin{algorithm}[t]
\SetAlgoCaptionSeparator{\ }
\caption{\textsc{Kinematic flow matching}}
\label{alg:kinematic_flow_matching}
\small
\DontPrintSemicolon
\SetKwInput{KwInput}{Input}
\SetKw{Return}{return}
\SetKwComment{Comment}{$\triangleright$\ }{}
\SetCommentSty{textit}
\SetNlSty{textnormal}{}{:\ }
\SetAlgoNlRelativeSize{0}
\KwInput{Coupling $\pi \in \Gamma(\mu_{t_0}, \ldots, \mu_{t_{J}})$, interpolant path $I_t$ and noise $G_t$}
\textbf{Initialize} acceleration network $a_\psi$\;
\While{{\rm training}}{
    Sample $t\sim\mathrm{Unif}(0,1)$\;
    Find $j$ such that $t\in(t_j,t_{j+1})$\;
    Draw $E_j=(Z_{t_j},Z_{t_{j+1}})\sim\pi_{j,j+1}$\;
    Draw $\xi\sim\mathcal N(0,I_{2d})$\; $X_t\gets I_t(E_j)+G_t\xi\,, V_t\gets\partial_t I_t(E_j)+\dot G_t\xi$\; $A_t\gets\partial_t^2 I_t(E_j)+\ddot G_t\xi$\; $\widehat{\cal L}_{\mathrm{KI}}\gets w(t)\|a_\psi(t,X_t,V_t)-A_t\|^2$\;$\psi\gets \texttt{optim}(\psi,\nabla_\psi{\widehat{\cal L}}_{\mathrm{KI}}(\psi))$\;
}
\Return $a_\psi\,.$\;
\end{algorithm}
\endgroup

\subsection{Sampling}
We now describe two sampling methods.

\looseness=-1
\paragraph{Deterministic sampling.}
We first consider discretizations of {\eqref{eq:markovian_ode}}, where we replace $a_t$ by its learned representation $a_\psi$. We initialize the dynamics with $(X_{0},V_{0})\sim\mu_{t_0}$ and numerically integrate \eqref{eq:markovian_ode} to $t_J$. For simplicity, we consider a standard discretization
\begin{align*}
    X_{(k+1)h} &= X_{kh}+h\,V_{kh}\,,\\
    V_{(k+1)h} &= V_{kh}+h \,a_\psi(kh,X_{kh},V_{kh})\,,
\end{align*}
for steps {$k=0,1,\ldots,K-1$}, and time step $h=1/K$.\looseness-1

\looseness=-1
\paragraph{Stochastic sampling.}
If we also trained a parametric model for the score $s_\phi$, then we can instead consider a discretization of \eqref{eq:background_kinetic_diffusion} with the drift in \eqref{eq:background_score_corrected_drift} that looks like
\begin{equation}
\label{eq:kinetic-sde-sampler}
\begin{split}
X_{(k+1)h} &= X_{kh}+hV_{kh}\,,\\
V_{(k+1)h} &= V_{kh}+h a_\psi(kh,X_{kh},V_{kh}) +\frac{h\beta_{kh}}{2}s_\phi(kh,X_{kh},V_{kh}) +\sqrt{h\beta_{kh}}\,\eta_k\,,
\end{split}
\raisetag{40pt}
\end{equation}
where $\eta_k \sim {\cal N}(0,I)$ are i.i.d.~and $\beta_t\geq0$ is a fixed sampling noise schedule.

\begin{table}[t]
\centering
\begingroup
\setlength{\tabcolsep}{6pt}
\renewcommand{\arraystretch}{1.10}
{\resizebox{\textwidth}{!}{\begin{tabular}{@{}lcccccc@{}}
& \multicolumn{4}{c}{Interpolation (small GoM)}
& \multicolumn{2}{c}{Forecast (large GoM)} \\
\cmidrule(lr){2-5}
\cmidrule(l){6-7}
Method & $t_1$ & $t_3$ & $t_5$ & $t_7$
& $t_{10}$ (from $t_0$) & $t_{10}$ (from $t_{9}$) \\
\midrule
\textsc{Vanilla-SB}{*} \citep{shen2025multimarginalschrodingerbridgesiterative}
& \vscore{0.112}{0.061} & \vscore{0.178}{0.073}
& \vscore{0.262}{0.043} & \vscore{0.305}{0.032}
& \textemdash & \textemdash \\
OT-CFM{*} \citep{tong2024improvinggeneralizingflowbasedgenerative}
& \vscore{0.148}{0.004} & \vscore{0.227}{0.008}
& \vscore{0.191}{0.012} & \vscore{0.250}{0.018}
& \textemdash & \textemdash \\
OT-MFM{*} \citep{kapusniak2024metricflowmatchingsmooth}
& \vscore{0.107}{0.014} & \vscore{0.056}{0.014}
& \vscore{0.052}{0.011} & \vscore{0.070}{0.021}
& \textemdash & \textemdash \\
WLM{*} \citep{guan2026lagrangianactionlearningpopulation}
& \vbest{0.039}{0.002} & \vscore{0.028}{0.000}
& \vscore{0.083}{0.002} & \vbest{0.050}{0.004}
& \vscore{0.180}{0.038} & \vscore{0.065}{0.005} \\
\midrule
OT-KI (Ours)
& \OurTtwo & \OurTfour & \OurTsix & \OurTeight
& \OurFromFirst & \OurFromLast \\
IND-KI (Ours)
& \OurIndependentTtwo & \OurIndependentTfour & \OurIndependentTsix & \OurIndependentTeight
& \OurIndependentFromFirst & \OurIndependentFromLast \\
\end{tabular}}}
\endgroup
\caption{Averaged $W_1$ distance for the GoM dataset; lower is better (*Prior results taken from~\cite{guan2026lagrangianactionlearningpopulation})}\label{tab:oceans-comparison}
\end{table}
\begin{figure}[tbp]
    \centering
    \includegraphics[width=\linewidth]{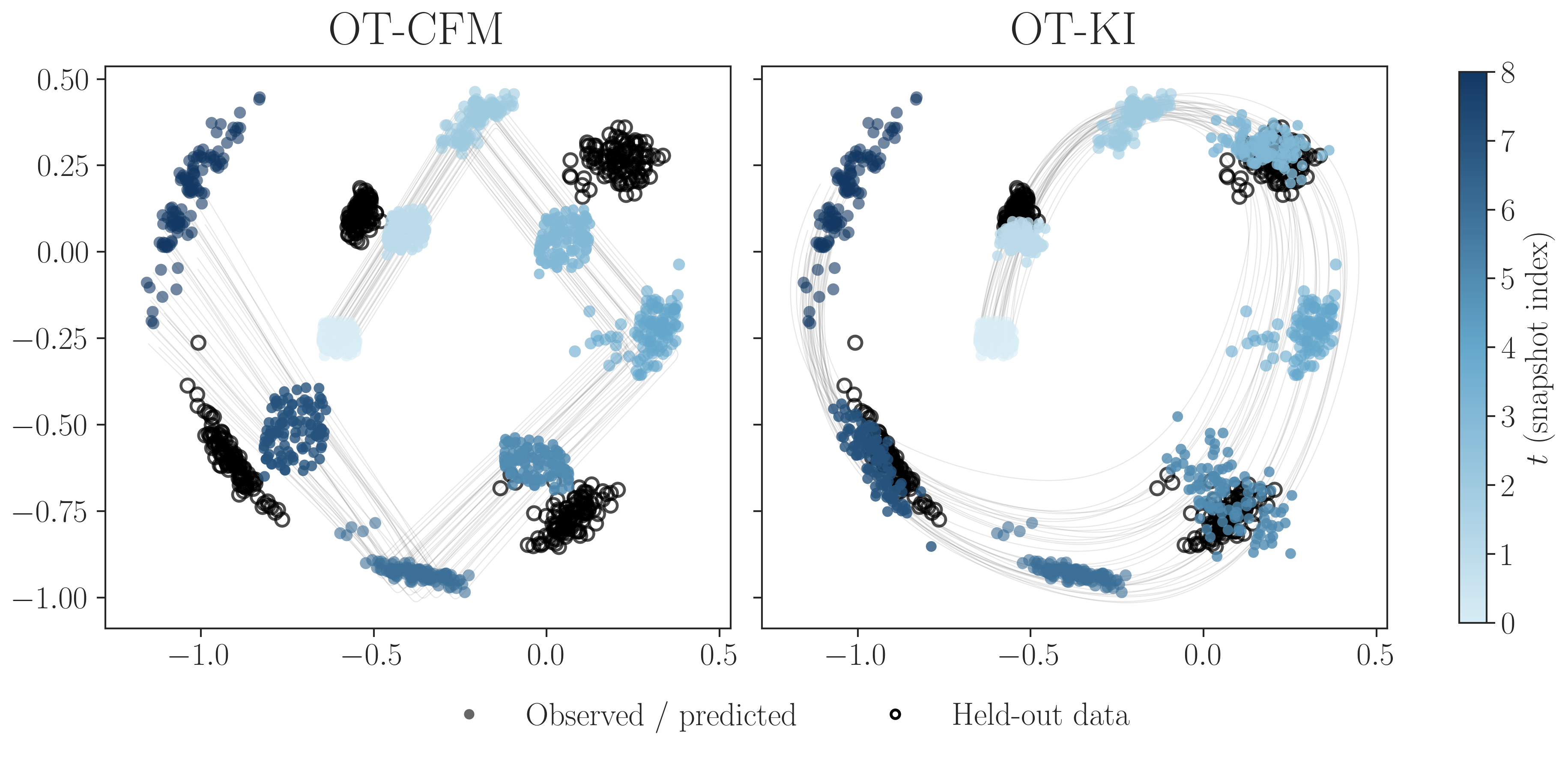}
    \caption{Trajectory rollouts for GoM (small vortex).}
    \label{fig:oceans-otcfm-otki}
\end{figure}

\begin{figure}
    \centering
    \includegraphics[width=1\linewidth]{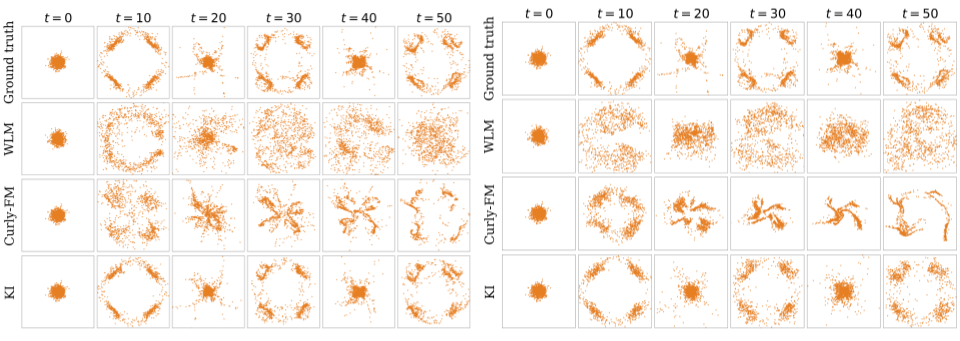}
    \caption{\textbf{Left:} Boids trained on dense snapshots. \textbf{Right:} Boids trained on sparse snapshots.}
    \label{fig:boids}
\end{figure}

\section{Numerical Experiments}\label{sec:experiments}
In this section we evaluate the performance of KI on three benchmark datasets from the literature. Among the methods we compare to, our closest competitors are Curly-FM~\citep{petrovic2025curlyflowmatchinglearning}, WLM~\citep{guan2026lagrangianactionlearningpopulation}, and OT-CFM~\citep{tong2024improvinggeneralizingflowbasedgenerative,pooladian2023multisample}. Omitted experimental details, including hyperparameters, are found in Appendix~\ref{app:exp_details}. Our GitHub repository can be found at \url{https://github.com/Submobile01/KinematicInterpolant}.
\subsection{Vortex datasets}
We consider two ocean vortex datasets derived from the HYbrid Coordinate Ocean Model Gulf of Mexico dataset from~\cite{shen2025multimarginalschrodingerbridgesiterative}, also known as GoM. We follow the setup of \cite{guan2026lagrangianactionlearningpopulation} and compare KI trained with minibatch optimal-transport coupling (OT-KI) and independent coupling (IND-KI). Table~\ref{tab:oceans-comparison} reports the $W_1$ distance between predicted and held-out positional distributions, with means and standard deviations over three seeds.

\subsubsection{Small vortex (interpolation)}
For the small vortex, we use the preprocessing of \citet{petrovic2025curlyflowmatchinglearning} and evaluate interpolation at the four held-out times $t_1,t_3,t_5,t_7$. Figure~\ref{fig:oceans-otcfm-otki} illustrates the benefit of incorporating velocities. Indeed, OT-CFM produces (essentially) straight paths between observed snapshots, cutting across the vortex, whereas OT-KI generated smoothed, curved trajectories more closely align with the held-out data. We observe that both our proposed KI variants improve on OT-CFM and OT-MFM at all four held-out times; in Table~\ref{tab:oceans-comparison} we see that IND-KI ultimately achieves the lowest mean $W_1$ at $t_3$ and $t_5$. In particular, IND-KI significantly reduces the error relative to WLM, which is the strongest baseline method.

\subsubsection{Large vortex (forecasting)}
The large vortex dataset follows the preprocessing of \citet{berlinghieri2025ohsnapmmdforecastingstochastic}. We predict the positional distribution at $t_{10}$, starting from either $t_0$ or $t_{9}$, as reported in the forecasting columns of Table~\ref{tab:oceans-comparison}. {OT-KI achieves the lowest mean $W_1$ for both forecasting initializations. Relative to WLM, the accuracy improves by approximately {$14\%$} when starting from $t_0$, and by over $55\%$ when starting from $t_9$.}

\subsection{Boids}
\looseness=-1
The Boids dataset is a two-dimensional synthetic benchmark of collective bird motion from \citet{guan2026lagrangianactionlearningpopulation}. We follow their data-generation procedure to simulate trajectories over $[0,50]$, focusing on interpolation rather than forecasting. We consider a short time horizon, $[0,24.5]$, and a long time horizon, $[0,50]$. For the dense experiments, we use snapshot spacings of $\delta t=0.5$ and $\delta t=1$, respectively, yielding approximately the same number of training marginals across horizons. In the sparse regime, we use a snapshot spacing of $\delta t=5$ for both horizons.
\begin{figure}
    \centering
    \includegraphics[width=0.7\linewidth]{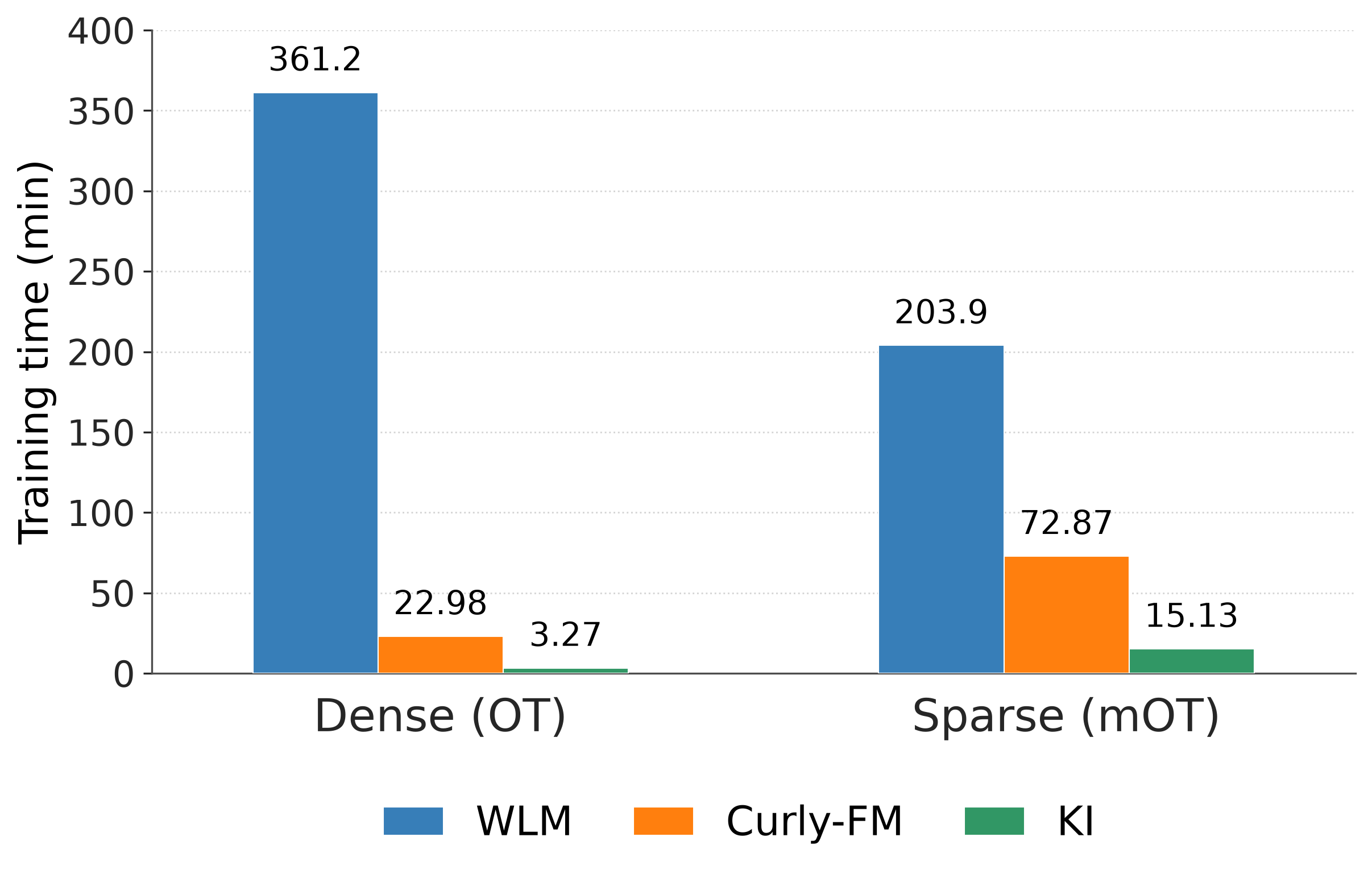}
    \caption{Training time for long-horizon Boids experiments on a single NVIDIA L40S GPU. Both Curly-FM and KI use precomputed full-population OT in the dense regime (51 training marginals) and per-step minibatch OT (mOT) in the sparse regime (11 training marginals). More details on training budgets are listed in Table~\ref{tab:boids_train_times_steps}.}
    \label{fig:training_time}
\end{figure}

\begin{figure}[ht]
    \centering
    \includegraphics[width=\linewidth]{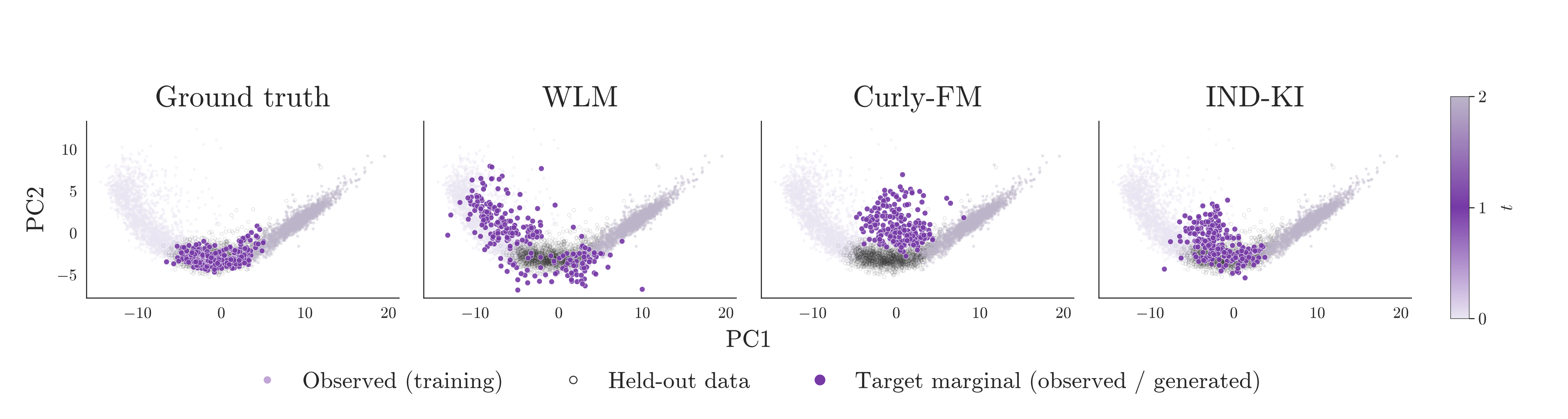}
    \caption{First two principal components of the generated marginal for 20-d mouse erythroid, comparing constant noise schedule KI using independent couplings with the ground truth and other methods.
    }
    \label{fig:mouse}
\end{figure}

\begin{table}[ht]
\centering
\small
\setlength{\tabcolsep}{4pt}
\resizebox{\linewidth}{!}{\begin{tabular}{lcccccc}
& \multicolumn{2}{c}{$d=2$}
& \multicolumn{2}{c}{$d=20$}
& \multicolumn{2}{c}{$d=50$} \\
\cmidrule(lr){2-3}
\cmidrule(lr){4-5}
\cmidrule(lr){6-7}
Method
    & CosSim $\downarrow$ & $W_2 \downarrow$
    & CosSim $\downarrow$ & $W_2 \downarrow$
    & CosSim $\downarrow$ & $W_2 \downarrow$  \\
\midrule
OT-CFM
    & $0.119 \pm 0.001$ & $0.611 \pm 0.006$
    & $0.326 \pm 0.007$ & $6.268 \pm 0.052$
    & $0.506 \pm 0.005$ & $7.888 \pm 0.035$ \\
Curly-FM
    & $\mathbf{0.009 \pm 0.000}$ & $0.369 \pm 0.090$
    & $0.335 \pm 0.001$ & $5.891 \pm 0.003$
    & $0.574 \pm 0.001$ & $7.733 \pm 0.011$ \\
WLM
    & $0.038 \pm 0.003$ & $0.416 \pm 0.026$
    & $0.373 \pm 0.067$ & $7.433 \pm 0.383$
    & $0.522 \pm 0.041$ & $9.011 \pm 0.173$ \\
\midrule
KI-bridge
    & $0.062 \pm 0.003$ & $0.288 \pm 0.038$
    & $0.267 \pm 0.014$ & $5.490 \pm 0.203$
    & $0.409 \pm 0.016$ & $7.130 \pm 0.060$ \\
KI-const
    & $0.072 \pm 0.003$ & $\mathbf{0.253 \pm 0.006}$
    & $\mathbf{0.234 \pm 0.003}$ & $\mathbf{5.271 \pm 0.128}$
& $\mathbf{0.373 \pm 0.009}$ & $\mathbf{6.804 \pm 0.007}$\\

\multicolumn{7}{r}{} \\
\end{tabular}}
\caption{Hold-out error for the mouse erythroid dataset in varying dimension.
}
\label{tab:erythroid-results}
\end{table}

\subsubsection{Dense snapshots with pre-computed OT}
For the dense experiments, we precompute OT couplings between neighboring snapshots using all $1000$ particles and the cost in~\eqref{eq:ot_cost}. We use KI with the bridge noise schedule, denoted KI-bridge. Table~\ref{tab:boids-wlm-am} shows that KI-bridge achieves the lowest mean $W_1$ and $W_2$ errors over both horizons, where its advantage over WLM is particularly pronounced over the long horizon.

Figure~\ref{fig:training_time} compares training times over the long horizon. In the dense regime, KI is approximately $100\times$ faster than WLM and $7\times$ faster than Curly-FM. In the sparse regime, where both KI and Curly-FM recompute minibatch OT couplings at each iteration, KI is approximately $13\times$ faster than WLM and $5\times$ faster than Curly-FM. Training budgets and timings of all regimes are reported in Table~\ref{tab:boids_train_times_steps} in Appendix~\ref{app:exp_details_boid}.
\looseness=-1
\subsubsection{Sparse snapshots with minibatch OT}
To assess performance when full-data couplings are unavailable or impractical, we evaluate minibatch OT with a batch size of $256$. The use of sparser training grids is motivated by the fact that minibatch-OT can produce ``mismatched'' endpoint pairs which lead to amplified errors. 

In this setting, we use KI-const with a constant noise schedule and recompute OT couplings within each minibatch at every update. Table~\ref{tab:boids-sparse} shows that KI-const achieves the lowest mean $W_1$ and $W_2$ errors among the methods trained on sparse grids, with a larger benefit over the baselines on the longer horizon.
\subsection{Mouse Erythroid}
The Mouse Erythroid dataset records cell differentiation along the erythroid lineage. Starting from the preprocessed data of \citet{petrovic2025curlyflowmatchinglearning}, we evaluate the methods in two-dimensional UMAP coordinates and in $20$- and $50$-dimensional PCA coordinates. Figure~\ref{fig:mouse} visualizes the results projected onto the first two principal components. As shown in Table~\ref{tab:erythroid-results}, KI with both constant and bridge noise schedules achieves lower mean $W_2$ errors than all evaluated baselines in every representation, with KI-const performing best throughout.  Details of preprocessing and the CosSim evaluation are provided in Appendix~\ref{app:mouse_erythroid}.

\begin{table}[t]
\centering
\begin{tabular}{lcccc}
& \multicolumn{2}{c}{$[0,24.5]$ (dense)}
& \multicolumn{2}{c}{$[0,50]$ (dense)} \\
\cmidrule(lr){2-3}
\cmidrule(lr){4-5}
Method & $W_1 \downarrow$ & $W_2\downarrow$ & $W_1\downarrow$ & $W_2\downarrow$ \\
\midrule
WLM
    & $0.594 \pm 0.188$ & $0.816 \pm 0.209$
    & $1.932 \pm 0.529$ & $2.268 \pm 0.542$ \\
Curly-FM
    & $0.995 \pm 0.173$ & $2.424 \pm 0.255$
& $2.180 \pm 0.254$ & $3.195 \pm 0.818$ \\
KI-bridge
    & $\mathbf{0.292 \pm 0.033}$ & $\mathbf{0.476 \pm 0.032}$
    & $\mathbf{0.450 \pm 0.062}$ & $\mathbf{0.682 \pm 0.076}$
\end{tabular}
\caption{Boids interpolation error with dense observations and full-data OT for KI.}

\label{tab:boids-wlm-am}
\end{table}

\begin{table}[t]
\centering
\begin{tabular}{lcccc}
& \multicolumn{2}{c}{$[0,24.5]$ (sparse)}
& \multicolumn{2}{c}{$[0,50]$ (sparse)}  \\
\cmidrule(lr){2-3}
\cmidrule(lr){4-5}
Method & $W_1\downarrow$ & $W_2\downarrow$ & $W_1\downarrow$ & $W_2\downarrow$ \\
\midrule
WLM
    & $0.639 \pm 0.036$ & $0.870 \pm 0.019$
    & $1.813 \pm 0.494$ & $2.131 \pm 0.552$ \\
Curly-FM
    & $1.476 \pm 0.108$ & $2.490 \pm 0.199$
& $2.564 \pm 0.294$ & $2.923 \pm 0.366$ \\
KI-const
 & $\mathbf{0.444 \pm 0.066}$ & $\mathbf{0.684 \pm 0.098}$
& $\mathbf{0.524 \pm 0.025}$ & $\mathbf{0.789 \pm 0.050}$ \\
\end{tabular}
\caption{Boids interpolation error with sparse observations and using minibatch OT for KI. }
\label{tab:boids-sparse}
\end{table}

\section{Related work}\label{sec:related_work}
\looseness=-1
\paragraph{Learning population dynamics.}
It is common to apply a least-action principle in order to find a unique solution favored by nature, as in optimal transport~\citep{schiebinger2019optimal,tong2024improvinggeneralizingflowbasedgenerative} and Schr\"odinger bridges with a Brownian reference~\citep{tong2023simulation}, or assume that the population follows a Wasserstein
gradient flow~\citep{bunne2022proximal,terpin2024learning}. In these cases, the inferred velocity
field is the gradient of a potential. However, gradient fields are curl-free, and they are fundamentally incapable of representing rotational motion such as the flow around an ocean vortex or progression through the cell cycle~\citep{petrovic2025curlyflowmatchinglearning}.

One line of work learns population dynamics through energy functionals on the space of probability distributions equipped with the Wasserstein distance. JKOnet~\citep{bunne2022proximal} learns an energy functional by differentiating through JKO steps~\citep{jordan1998variational}. JKOnet$^*$~\citep{terpin2024learning} avoids differentiating through simulations and instead minimizes residuals of the discrete JKO optimality conditions. Wasserstein Residuals~\citep{heinonen2026wassersteinresidualslearninggradient} adopts a continuous-time residual formulation for learning Wasserstein gradient flows, which additionally enables simulation-free training.

Beyond first-order gradient flows, Wasserstein Lagrangian Mechanics (WLM)~\citep{guan2026lagrangianactionlearningpopulation} models interacting second-order dynamics through a (damped) Newtonian equation. WLM learns a population-level potential energy by differentiating through simulated trajectories to minimize discrepancies between predicted and observed marginals. Wasserstein Lagrangian Residuals~\citep{sergeev2026simulationfreelearningpopulationdynamics} extends the residual approach to Wasserstein Lagrangian dynamics, enabling simulation-free training.

\looseness=-1
\paragraph{Flow and bridge matching approaches.}
Several flow matching approaches adopt a two-stage procedure, first learning an interpolant or coupling between observed marginals and then training a flow through regression. Curly-FM~\citep{petrovic2025curlyflowmatchinglearning} also incorporates position--velocity observations; they construct a continuous reference velocity field from discrete observations using a prescribed kernel,
then learn an interpolant aligned with this field before performing flow matching.
KI instead imposes the observed endpoint velocities directly as boundary conditions on its interpolants.

Most other approaches use only positional observations. For instance, Metric Flow Matching~\citep{kapusniak2024metricflowmatchingsmooth} learns an interpolant from positional observations using a data-dependent metric.
\citet{wang2025optimaltransportdensitydependentcost} also discourage trajectories from passing through low-density regions by using a density-dependent transport cost.
\citet{schwerdtner2026twoparameterflowslearningpopulation} first propose to learn a flow from a common reference distribution to each positional marginal, conditioned on physical time. They then extract the induced physical-time velocity through regression, using derivatives of the learned transport with respect to physical time.
SBIRR~\citep{shen2025multimarginalschrodingerbridgesiterative} alternates between estimating a multi-marginal Schr\"odinger bridge under a reference process and refining the reference within a prescribed (parametric) family. They use a maximum-likelihood loss to fit simulated bridge trajectories.

Among position-only methods, Acceleration Matching (AM)~\citep{dazzini2026trajectory} also learns a Markovian acceleration field in phase space. AM samples the unobserved velocities at the observation times by conditioning its kinetic Brownian reference process on the positional observations, using closed-form Gaussian conditioning. Relatedly, TracingFlow~\citep{sun2026tracingflowsimulationfreetrajectoryinference} instead iteratively assigns velocities at observed positions using the conditional mean of the velocities of trajectories passing through those positions.

\looseness=-1
\paragraph{Higher-order generative modeling.}
Higher-order dynamics have also been explored for generative modeling. AGM~\citep{chen2024generativemodelingphasestochastic} constructs stochastic bridges in phase space from noise to data, using an auxiliary velocity variable to promote straighter trajectories. Its acceleration targets follow from a stochastic optimal-control formulation and the associated probability-flow dynamics. TADA~\citep{chen2025tadaimproveddiffusionsampling} further explores augmented dynamics for efficient generative sampling, enabling pretrained diffusion models to be reused within momentum-based ODE samplers without additional training. \citet{jiang2026kineflow} propose a second-order flow for time-series forecasting.

\section{Conclusion}
We introduced Kinematic Interpolants, a flexible framework for reconstructing population dynamics from unpaired, time-indexed snapshots of positions and velocities. By directly learning acceleration fields through simulation-free training, the framework accommodates multiple noise schedules. Experiments on ocean vortices, mouse erythroid differentiation, and collective bird motion demonstrate improved reconstruction accuracy and substantially lower training costs than the compared baselines.

Our current formulation does not explicitly encode physical priors, and our experiments largely focus on interpolation. Future work can incorporate physical structure, investigate time-independent acceleration fields for forecasting, explore mean interpolants beyond cubic Hermite polynomials, and consider latent (potentially learnable) noise schedules.

\subsection*{AI Use Statement}
AI tools were used to assist with formatting, code implementation, drafting descriptions of implementation details in the appendix, and copy-editing. All AI-assisted content in the paper was reviewed and verified by the authors.

\subsubsection*{Acknowledgements}
The authors would like to thank the ``Diffusion Generative Modeling: Progress and Next Steps'' workshop hosted by the Simons Institute at UC Berkeley, and the Generative Models and Sampling workshop hosted by Microsoft Research New England, for creating stimulating environments that led to this collaboration. This material is based upon work supported by the U.S.
National Science Foundation Graduate Research
Fellowship under Grant No. 2141064. Any
opinions, findings, and conclusions or recommendations expressed in this material are those of the author(s) and do not necessarily reflect the views of the National Science Foundation.

\bibliographystyle{plainnat}
\bibliography{\arxivpath bibliography}

\clearpage
\appendix
\section{Proofs}
\subsection{Proof of Proposition~\ref{prop:kinetic_transport}}
\label{app:kinetic_transport_pf}
\begin{proof}
Consider a smooth test function $\varphi\in C_c^\infty(\mathbb R^{2d})$. By the chain rule and the definition of the conditional acceleration, differentiation in time on each interval $(t_j,t_{j+1})$ gives
\begin{align*}
\frac{\dd}{\dd t}\mathbb E[\varphi(Z_t)]
&=\mathbb E[\nabla_x\varphi(Z_t)\cdot V_t
+\nabla_v\varphi(Z_t)\cdot A_t^{(j)}]\\
&=\mathbb E\bigl[\nabla_x\varphi(Z_t)\cdot V_t
+\nabla_v\varphi(Z_t)\cdot\mathbb E\bigl[A_t^{(j)}\mid Z_t\bigr]\bigr]\\
&=\mathbb E[\nabla_x\varphi(Z_t)\cdot V_t
+\nabla_v\varphi(Z_t)\cdot a_t^M(Z_t)]\,.
\end{align*}
This implies that the interpolant density satisfies the continuity equation
\begin{equation}
\label{eq:kinetic_continuity}
\partial_t\rho_t+\nabla_x\cdot(v\rho_t)+\nabla_v\cdot(a_t^M\rho_t)=0\,.
\end{equation}

The ODE~\eqref{eq:markovian_ode} satisfies the same continuity equation with the same initial density, so uniqueness under standard smoothness assumptions gives~\eqref{eq:markovian_marginal_matching}. The endpoint constraints ensure that the interpolant density matches the known densities at the prescribed times, i.e., $\rho_{t_j}=\mu_{t_j}$ which proves the claim.
\end{proof}

\section{Acceleration Matching as a special case}\label{sec:am_special_case}

Fix an interval $[t_j,t_{j+1}]$ and write $\delta=t_{j+1}-t_j$ and $u=(t-t_j)/\delta$. We now briefly go over the construction from Acceleration Matching~\citep{dazzini2026trajectory}. In that work, the authors consider the (undamped) reference process
\begin{equation}
\label{eq:am_reference_sde}
    dX_t=V_t\,dt\,,
    \qquad
    dV_t=\sqrt{\epsilon}\,dB_t\,,
\end{equation}
where the idea is to condition on neighboring endpoints $E_j=(Z_{t_j},Z_{t_{j+1}})$ and regress onto the resulting kinetic Brownian bridges. (Crucially, their work assumes no velocity data, and so it must be sampled when training; see Section~\ref{sec:related_work} for more information.) The following proposition shows that our formulation recovers theirs.

\begin{mdframed}[style=MyFrame2]
\begin{proposition}[Recovering the AM bridge marginals]
\label{prop:am_si_equivalence}
Choose the Hermite interpolant $I_t=H_t$ and
\begin{equation*}
G_t=\sqrt{\epsilon}\,(\ell_t\otimes I_d)\,,\qquad
\ell_t={\gamma_t}\begin{pmatrix}\cos\theta_t&\sin\theta_t\end{pmatrix}\,,
\end{equation*}
with, for $0<u<1$,
\begin{equation}
\label{eq:am_rho_theta}
\begin{aligned}
\!\!\!\!\!\!\!\!\!{\gamma_t}=\frac{\delta^{3/2}}{\sqrt{3}}[u(1-u)]^{3/2}\,,\, \theta_t=\frac{\sqrt{3}}{2}\log\Bigl(\frac{u}{1-u}\Bigr)\,.
\end{aligned}
\end{equation}
Then the interpolant and the AM bridge \eqref{eq:am_reference_sde} have the same conditional marginals, i.e., for $t \in [t_j,t_{j+1}]$
\begin{equation}
\label{eq:conditional_marginal_equivalence}
\operatorname{Law}(Z_t\mid E_j)
=\operatorname{Law}(Z_t^{\mathrm{AM}}\mid E_j)\,.
\end{equation}
\end{proposition}

\end{mdframed}

\subsection{Proof of Proposition~\ref{prop:am_si_equivalence}}
\label{app:am_si_equivalence_pf}

Fix an interval $[t_j,t_{j+1}]$ and endpoints $E_j=(z_0,z_1)$.
Write $s=t-t_j$, $r=t_{j+1}-t$, and
$\delta=t_{j+1}-t_j=s+r$, so
$s=\delta u$ and $r=\delta(1-u)$ in \eqref{eq:am_rho_theta}.

\begin{lemma}[Gaussian law of the AM bridge]
\label{lem:am_bridge_law}
The conditional law of \eqref{eq:am_reference_sde} is
\begin{equation}
\label{eq:am_bridge_gaussian}
Z_t^{\mathrm{AM}}\mid E_j\sim
\mathcal N\!\left(
\begin{pmatrix}H_t(E_j)\\\partial_t H_t(E_j)\end{pmatrix},
\epsilon K_t^{\mathrm{AM}}\otimes I_d\right)\,,
\end{equation}
where
\begin{equation}
\label{eq:am_covariance}
K_t^{\mathrm{AM}}=
\begin{pmatrix}
\dfrac{s^3r^3}{3\delta^3} & \dfrac{s^2r^2(r-s)}{2\delta^3}\\[3mm]
\dfrac{s^2r^2(r-s)}{2\delta^3} & \dfrac{sr(s^2-sr+r^2)}{\delta^3}
\end{pmatrix}\,.
\end{equation}
\end{lemma}
\begin{proof}
Without loss of generality, it suffices to work in $d=1$. For all $h\geq 0$, integrating \eqref{eq:am_reference_sde} over an interval of length $h$ gives
\[
\begin{pmatrix}X_{t_j+h}\\V_{t_j+h}\end{pmatrix}
=\begin{pmatrix}X_{t_j}+hV_{t_j}\\V_{t_j}\end{pmatrix}
+\sqrt{\epsilon}\int_0^h\begin{pmatrix}h-q\\1\end{pmatrix}\dd W_q\,,
\]
where $W_q=B_{t_j+q}-B_{t_j}$ is a standard Brownian motion independent of the initial state. The deterministic term identifies the transition matrix
\begin{align*}
    F(h)= \begin{pmatrix}
    1 & h\\
    0 & 1
\end{pmatrix}\,,
\end{align*}
while It\^o's isometry gives
\begin{align*}
    Q(h)=\int_0^h
\begin{pmatrix}(h-q)^2&h-q\\h-q&1\end{pmatrix}\,dq
=\begin{pmatrix}h^3/3&h^2/2\\h^2/2&h\end{pmatrix}\,.
\end{align*}
The cross-covariance below follows by writing the final state as $F(r)Z_t$ plus an independent noise term.
Now given $Z_{t_j}=z_0$, the state at time $t$ has mean $F(s)z_0$, covariance $\epsilon Q(s)$, and cross-covariance $\epsilon Q(s)F(r)^\top$ with the final state. Performing another Gaussian conditioning with $Z_{t_{j+1}}=z_1$ then gives\looseness-1
\begin{align}
\label{eq:gauss_cond_mean}
m_t&=F(s)z_0+Q(s)F(r)^\top Q(\delta)^{-1}
\bigl(z_1-F(\delta)z_0\bigr)\,,\\
\label{eq:gauss_cond_cov}
K_t^{\mathrm{AM}}&=Q(s)-Q(s)F(r)^\top Q(\delta)^{-1}F(r)Q(s)\,.
\end{align}
To make the final calculation explicit, set $u=s/\delta$. The inverse endpoint covariance and conditioning matrix are
\begin{align*}
    Q(\delta)^{-1}=\begin{pmatrix}12/\delta^3&-6/\delta^2\\-6/\delta^2&4/\delta\end{pmatrix}\,,\qquad
C_t\defeq Q(s)F(r)^\top Q(\delta)^{-1}
=\begin{pmatrix}
u^2(3-2u)&-\delta u^2(1-u)\\
6u(1-u)/\delta&u(3u-2)
\end{pmatrix}\,.
\end{align*}
Writing $z_k=(x_k,v_k)$, the first component of $m_t$ is therefore
\begin{align*}
    (m_t)_x=(2u^3-3u^2+1)x_0+\delta(u^3-2u^2+u)v_0
+(-2u^3+3u^2)x_1+\delta(u^3-u^2)v_1\,,
\end{align*}
which is precisely the cubic Hermite interpolant $H_t(E_j)$. Meanwhile, the second row of the conditioning formula gives its time derivative (which uses that $\tfrac{\dd u}{\dd t}=\delta^{-1}$). For the covariance, substituting $C_t$ into $Q(s)-C_tF(r)Q(s)$ yields
\begin{align*}
    (K_t^{\mathrm{AM}})_{11}&=\frac{\delta^3}{3}u^3(1-u)^3\,,\\
(K_t^{\mathrm{AM}})_{12}&=\frac{\delta^2}{2}u^2(1-u)^2(1-2u)\,,\\
(K_t^{\mathrm{AM}})_{22}&=\delta u(1-u)(1-3u+3u^2)\,.
\end{align*}
Replacing $\delta u$ by $s$ and $\delta(1-u)$ by $r$ gives \eqref{eq:am_covariance}.
\end{proof}

\begin{lemma}[Noise factor for the AM covariance]
\label{lem:am_noise_factor}
The schedule \eqref{eq:am_rho_theta} and factor \eqref{eq:polar_l} satisfy
\begin{equation}
\label{eq:am_square_root_identity}
\Lambda_t\Lambda_t^\top=K_t^{\mathrm{AM}}\,,
\end{equation}
where $\Lambda_t\defeq\begin{pmatrix}\ell_t\\ \dot\ell_t\end{pmatrix}\in\R^{2\times2}$.
\end{lemma}
\begin{proof}
{Let $e_t=(\cos\theta_t \,\, \sin\theta_t)$ and $e_t^\perp=(-\sin\theta_t \, \, \cos\theta_t)$. These are orthonormal, and
\begin{align*}
    \ell_t=\gamma_t e_t\,,\qquad
\dot\ell_t=\dot\gamma_t e_t+\gamma_t\dot\theta_t e_t^\perp\,,
\end{align*}
from which we can directly compute
\begin{align*}
\Lambda_t\Lambda_t^\top=
\begin{pmatrix}
\gamma_t^2&\gamma_t\dot\gamma_t\\
\gamma_t\dot\gamma_t&\dot\gamma_t^2+\gamma_t^2\dot\theta_t^2
\end{pmatrix}\,.
\end{align*}
In terms of the variables $s,r$, we have that the schedule is $\gamma_t=(sr)^{3/2}/(\sqrt{3}\,\delta^{3/2})$ and $\theta_t=(\sqrt{3}/2)\log(s/r)$, with $\dot s=1$ and $\dot r=-1$.}
Differentiating the schedule via chain rule then gives
\begin{align*}
    \dot\gamma_t=\frac{\sqrt{3sr}(r-s)}{2\delta^{3/2}}\,,\qquad
\dot\theta_t=\frac{\sqrt{3}\,\delta}{2sr}\,,
\end{align*}
and consequently,
\begin{align*}
    \gamma_t^2=\frac{s^3r^3}{3\delta^3}\,,\qquad
\gamma_t\dot\gamma_t=\frac{s^2r^2(r-s)}{2\delta^3}\,,\qquad
\dot\gamma_t^2+\gamma_t^2\dot\theta_t^2
=\frac{sr(s^2-sr+r^2)}{\delta^3}\,,
\end{align*}
which proves the identity.
\end{proof}
\begin{proof}[Proof of Proposition~\ref{prop:am_si_equivalence}]
Conditioned on $E_j$, the interpolant is Gaussian with mean $(H_t(E_j),\partial_tH_t(E_j))$ and covariance $\epsilon(\Lambda_t\Lambda_t^\top)\otimes I_d$. The two lemmas identify these with the mean and covariance of the AM bridge. This proves equality of the conditional laws in the interval interior; the endpoint extension gives the common prescribed endpoints.
\end{proof}

\subsection{Damped bridge mean and covariance}
\label{sec:damped_bridge}

The same construction applies to the damped kinetic reference process
\[
\mathrm dX_t=V_t\,\mathrm dt,\qquad
\mathrm dV_t=-\alpha V_t\,\mathrm dt
+\sqrt{\epsilon}\,\mathrm dB_t,
\qquad \alpha,\epsilon>0.
\]
Fix endpoints $E_j=(z_0,z_1)$, where
$z_0=(x_{t_j},v_{t_j})$ and
$z_1=(x_{t_{j+1}},v_{t_{j+1}})$, and write
$s=t-t_j$, $r=t_{j+1}-t$, and
$\delta=t_{j+1}-t_j=s+r$.
For each spatial coordinate, the transition matrix and
noise covariance over a duration $h$ are $F_\alpha(h)$
and $\epsilon Q_\alpha(h)$, with
\begin{align*}
F_\alpha(h)=
\begin{pmatrix}
1&\dfrac{1-e^{-\alpha h}}{\alpha}\\[2mm]
0&e^{-\alpha h}
\end{pmatrix}\,,\quad
Q_\alpha(h)=
\begin{pmatrix}
\dfrac{2\alpha h-3+4e^{-\alpha h}-e^{-2\alpha h}}
{2\alpha^3}
&
\dfrac{(1-e^{-\alpha h})^2}{2\alpha^2}\\[3mm]
\dfrac{(1-e^{-\alpha h})^2}{2\alpha^2}
&
\dfrac{1-e^{-2\alpha h}}{2\alpha}
\end{pmatrix}\,.
\end{align*}
Applying the Gaussian conditioning formulas
\eqref{eq:gauss_cond_mean}--\eqref{eq:gauss_cond_cov}
with $F_\alpha,Q_\alpha$ gives the coordinatewise mean
\begin{equation*}
m_\alpha(t)
=
F_\alpha(s)z_0
+
Q_\alpha(s)F_\alpha(r)^\top Q_\alpha(\delta)^{-1}
\bigl(z_1-F_\alpha(\delta)z_0\bigr)
\end{equation*}
where, as before, this can be expressed in closed-form, and the covariance $\epsilon K_\alpha(t)$ is given by
\begin{align*}
K_\alpha(t)
=
Q_\alpha(s)
-
Q_\alpha(s)F_\alpha(r)^\top Q_\alpha(\delta)^{-1}
F_\alpha(r)Q_\alpha(s).
\end{align*}

\section{Experimental Details}
\label{app:exp_details}
\paragraph{Noise scaling with interval length.}
For an observed interval of normalized duration $h$, we define
the position-noise envelopes
\begin{equation}\label{eq:noise_schedule_params}
a_{\mathrm{const}}(\tau)
= \frac{\sigma h^{3/2}}{8\sqrt{3}},
\qquad
a_{\mathrm{bridge}}(\tau)
= \frac{2\sigma h^{3/2}}{\sqrt{3}}\tau^2(1-\tau)^2,
\end{equation}
where $\tau\in[0,1]$ is local interval time and $\sigma\geq0$
controls the noise magnitude. This normalization matches the
constant noise amplitude to the peak bridge noise amplitude.
For both schedules, the rotation angle is
$\theta(\tau)=\omega\tau$, with $\omega>0$, and the position
and velocity noise standard deviations are
\begin{equation}
s_x(\tau)=a(\tau),
\qquad
s_v(\tau)=\frac{1}{h}
\sqrt{[a'(\tau)]^2+\omega^2a(\tau)^2}.
\end{equation}
For fixed hyperparameters, these scale as $h^{3/2}$ and
$h^{1/2}$, respectively. KI-bridge has zero position and velocity
noise at the endpoints, whereas KI-const retains nonzero
endpoint noise when $\sigma>0$.
\subsection{Oceans (Gulf of Mexico)}
\label{app:oceans}
\paragraph{{Data and preprocessing.}}
We use the two preprocessed Gulf of Mexico vortex datasets described in the main text from~\cite{shen2025multimarginalschrodingerbridgesiterative}. The small vortex contains nine snapshots of $111$ particles, and the large vortex contains eleven snapshots of $400$ particles; each particle has a two-dimensional position and velocity. For interpolation, we observe snapshots $t_0\,,t_2\,,t_4\,,t_6\,,t_8$ and hold out $t_1\,,t_3\,,t_5\,,t_7$. For forecasting, we observe $t_0\,,\ldots\,,t_9$ and hold out $t_{10}$. Time is rescaled so that the last observed snapshot is at $1$, and velocities are rescaled consistently with both space and time.

\paragraph{{Training.}}
Both KI variants use cubic Hermite mean paths and constant-amplitude rotating Gaussian noise. OT-KI recomputes minibatch OT couplings at each iteration whereas IND-KI pairs the two minibatches independently. We use separate networks to predict the acceleration field and the velocity score. We use exponentially averaged weights to improve stability of training; Table~\ref{tab:oceans_hyperparameters} summarizes the shared model and training settings.

OT-CFM starts at the observed positions and uses a position-only (with time embeddings) velocity network with the same hidden width, depth, and training budget. We use the noise level consistent with the constant noise schedule in~\eqref{eq:noise_schedule_params} with {$\sigma=1\,,\omega=\pi\sqrt{3}$}, $h=1/4$ for the small vortex and $h=1/9$ for the large vortex.

\paragraph{Evaluation.}
{Our results in Table~\ref{tab:oceans-comparison} use score-corrected SDE sampling with velocity diffusion coefficient of $\kappa=0.5$ (equivalently, $\beta_t=2\kappa=1$ in~\eqref{eq:kinetic-sde-sampler}), corresponding to score drift $\kappa s_t$ and Brownian amplitude $\sqrt{2\kappa}$. As a numerical integrator, we use the {Euler--Maruyama} method with $125$ steps per original snapshot interval. OT-CFM uses {Heun ODE sampling} with the same integration step size.}

Following~\cite{guan2026lagrangianactionlearningpopulation}, the small-vortex results in Table~\ref{tab:oceans-comparison} use separate rollouts from each preceding observed snapshot to the next held-out time, initialized at the observed positions and velocities without added noise. The large-vortex forecasts start at either $t_0$ or $t_9$ and run to $t_{10}$, with independent Gaussian perturbations of standard deviations $\sigma_x$ and $\sigma_v$ added at initialization. We compute the $1$-Wasserstein distance between the predicted and held-out position clouds in the original dataset coordinates, using all particles, and report means and sample standard deviations over three seeds.

\subsection{Boids}
\label{app:exp_details_boid}
\paragraph{Noise scaling.}
We use $\sigma(h)=\sigma_0\sqrt{h_{\mathrm{ref}}/h}$ to keep
velocity-noise amplitudes independent of interval length.
For KI-const, $(\sigma_0,\omega,h_{\mathrm{ref}})$ is
$(0.2,2,0.98)$ for the short full grid,
$(2/50^{3/2},\pi\sqrt{3},2)$ for the long full grid,
and $(0.1,2,0.98)$ for both sparse grids.
Here, interval durations are in physical time.
The SDE sampling amplitudes are $1$, $8/50^{3/2}$,
and $0.2$, respectively.
For the short-horizon KI-bridge preset, we use
$(\sigma_0,\omega,h_{\mathrm{ref}})=(0.3,10,0.98)$,
with envelope multiplier $1$ and SDE sampling amplitude $1$.
\paragraph{Hyperparameters.} We use zero friction for WLM and follow its recommended $30{,}000$ training steps on the full grids. All step budgets remain fixed across horizons. Regarding integration steps during training, WLM uses a fixed training integration step of $0.5$, requiring
1 and 2 steps per observed interval in the short- and long-horizon
full settings, respectively, and 10 in both sparse settings.

Curly-FM uses $5{,}000$ geometry-training steps followed by
$5{,}000$ flow-training steps. Both Curly-FM and KI use a batch size
of $256$ per training interval and include all intervals in each
optimization step. For KI, each step samples one time from every
interval. With interval weights matching the time-sampling
distribution, this estimator has the same expectation as the
single-time estimator in the original algorithm.

Both methods use
precomputed full-population OT in the dense regime and recompute
minibatch OT at each step in the sparse regime. For Curly-FM, these
couplings are used during flow training; dense couplings are computed
after freezing the geometry network, using the learned curve cost
averaged over eight time points per interval. Its velocity prior
uses $100$ nearest neighbors.

Table~\ref{tab:boids_train_times_steps} reports training budgets and wall-clock times. KI constructs cost matrices using a closed-form acceleration cost; recall \textbf{ref}.  Curly-FM, on the other hand, evaluates (and subsequently differentiates) a neural network that learns the geometry from the data, and queries a nearest-neighbor velocity prior for each candidate endpoint pair. This increases the total computational cost when the matrix is being reinstantiated at every flow-training step. The size of the geometry network therefore affects the runtime of both geometry learning and subsequent flow training.
\begin{table}[t]
\centering
\caption{Boids training budgets and wall-clock times on a single
NVIDIA L40S GPU and 8 CPUs, excluding unrelated setup and evaluation.
Curly-FM step counts combine geometry and flow training.
KI and Curly-FM timings include OT precomputation in the dense regime and
per-step minibatch OT in the sparse regime.}
\label{tab:boids_train_times_steps}
\resizebox{\linewidth}{!}{\begin{tabular}{lrrrrrrrr}
& \multicolumn{4}{c}{Dense}
& \multicolumn{4}{c}{Sparse} \\
\cmidrule(lr){2-5}\cmidrule(lr){6-9}
& \multicolumn{2}{c}{Short}
& \multicolumn{2}{c}{Long}
& \multicolumn{2}{c}{Short}
& \multicolumn{2}{c}{Long} \\
\cmidrule(lr){2-3}\cmidrule(lr){4-5}
\cmidrule(lr){6-7}\cmidrule(lr){8-9}
Method & Steps & Time (min) & Steps & Time (min)
       & Steps & Time (min) & Steps & Time (min) \\
\midrule
WLM
& $30,000$ & $276.3$
& $30,000$ & $361.2$
& $20,000$ & $107.8$
& $20,000$ & $203.9$ \\
Curly-FM
& $10,000$ & $22.44$ & $10,000$ & $22.98 $

& $10,000$ & $35.72 $ & $10,000$ & $72.87 $ \\
KI
& $20,000$ & $3.10$
& $20,000$ & $3.27$
& $20,000$ &  $7.86$
& $20,000$ & $15.13 $ \\
\end{tabular}}
\end{table}
\subsection{Mouse Erythroid}
\label{app:mouse_erythroid}
\paragraph{Noise parameters.}
For the PCA experiments, we use $\omega=6$ and set
$\sigma=0.2$ for KI-const and $\sigma=2$ for KI-bridge.
For the two-dimensional UMAP experiments, we use
$\omega=\pi\sqrt{3}$ and $\sigma=2$ for both schedules.
\paragraph{Velocity projections.}
We found that direct PCA projection empirically produced poorly balanced velocity components across the retained dimensions. We therefore use the graph-based velocity embedding of scVelo~\citep{Bergen820936}, setting \texttt{direct\_pca\_projection=False} in \texttt{scvelo.tl.velocity\_embedding}, whereas the Curly-FM setup uses \texttt{direct\_pca\_projection=True}. We rerun all baselines using the same graph-based projections to ensure consistent preprocessing across methods.

\paragraph{Velocity calibration.}
The embedded RNA velocities provide directional information, but their magnitudes require calibration to the time scale of the observed marginals \citep{lamanno2018rna, Bergen820936}. We apply a global velocity rescaling based on the displacement of the mean between training marginals per unit time.
\paragraph{RNA-velocity filtering.}
For cosine evaluation, we follow the RNA-velocity filtering approach of \citet{petrovic2025curlyflowmatchinglearning}, with a modified noise-mixing weight. Let $\bar v_k(x)$ denote the Gaussian-weighted average of the velocities at the $k$ nearest neighbors of $x$, with bandwidth equal to the distance to the farthest selected neighbor. The filtered velocity is
\[
\widetilde v(x)
=
(1-\alpha(x))\bar v_k(x)+\alpha(x)\varepsilon,
\qquad
\varepsilon\sim\mathcal{N}(0,0.1I_d),
\]
where $\alpha(x)=\sigma(3(d_{\min}(x)-5))$, $\sigma$ is the logistic sigmoid, and $d_{\min}(x)$ is the distance to the nearest observation. We set the distance threshold to $5$ and the sigmoid steepness to $3$, giving greater weight to noise farther from the observed data. Compared with the original mixing rule, $\alpha_{\mathrm{Curly}}(x)=\tfrac12+\tfrac12\sigma(100(d_{\min}(x)-0.2))$, which assigns at least half the weight to noise at every query point, our rule preserves more of the estimated velocity signal near the observed data while gradually attenuating it in more distant regions.

For the two-dimensional UMAP experiments, however, we omit noise mixing and use $\widetilde v(x)=\bar v_k(x)$, following \citet{petrovic2025curlyflowmatchinglearning}.

The cosine distance (CosSim on table \ref{tab:erythroid-results}) used by Curly-FM is
\[
D_{\mathrm{cos}}
=
\frac{1}{N(t_b-t_a)}
\sum_{i=1}^{N}\int_{t_a}^{t_b}
\left[1-\cos\!\left(\hat v_i(t),
\widetilde v(\hat x_i(t))\right)\right]\,dt,
\]
which is numerically approximated using the trapezoidal rule. We use $k=20$ and $[t_a,t_b]=[0,1]$ for 2D UMAP experiments, and $k=30$ and $[t_a,t_b]=[0,0.5]$ for 20D and 50D PCA experiments. Lower cosine distance implies better alignment.

\begin{table}[t]
\centering
\small
\caption{Training and held-out evaluation marginals for the Boids
experiments. Within each time range, both coupling strategies are
evaluated on the same test marginals. Each marginal contains
$N=1000$ particles.}
\begin{tabular}{lllc}
\toprule
Range & Coupling & Training times (s) & Number \\
\midrule
Short & Full-batch OT
      & $\{0,0.5,\ldots,24.5\}$ & 50 \\
Short & Minibatch OT
      & $\{0,5,10,15,20,24.5\}$ & 6 \\
Long  & Full-batch OT
      & $\{0,1,\ldots,50\}$ & 51 \\
Long  & Minibatch OT
      & $\{0,5,\ldots,50\}$ & 11 \\
\midrule
\multicolumn{2}{l}{Short-range test times}
      & $\{0.25,0.75,\ldots,24.25\}$ & 49 \\
\multicolumn{2}{l}{Long-range test times}
      & $\{0.5,1.5,\ldots,49.5\}$ & 50 \\
\bottomrule
\end{tabular}
\label{tab:boids_marginals}
\end{table}

\begin{table}[t]
\centering
\small
\resizebox{\linewidth}{!}{\begin{tabular}{lcccc}
\toprule
& \multicolumn{2}{c}{KI}
& \multicolumn{2}{c}{Curly-FM} \\
\cmidrule(lr){2-3}\cmidrule(lr){4-5}
Hyperparameter & Bridge & Constant & Dense & Sparse \\
\midrule
OT coupling
& Full-batch OT & Minibatch OT & Full-batch OT & Minibatch OT \\
OT coupling population
& $1000$ & $256$ & $1000$ & $256$ \\
Coupling construction
& Precomputed & Per update & Precomputed & Per update \\
Sampling method
& Score-corrected SDE & Score-corrected SDE & ODE & ODE \\
MLP hidden width
& $512$ & $512$ & $256$ & $256$ \\
MLP hidden depth
& $4$ & $4$ & $3$ & $3$ \\
Learning rate
& $10^{-3}$ & $3\times10^{-4}$ & $10^{-4}$ & $10^{-4}$ \\
\bottomrule
\end{tabular}}
\caption{Boids model and training settings.
Curly-FM width, depth, and learning rate apply to both
geometry and flow training. KI bridge settings correspond
to the 50-second long-poly run.}
\label{tab:boids_hyperparameters}
\end{table}

\begin{table}[t]
\centering
\small
{\begin{tabular}{lll}
\toprule
Hyperparameter & OT-KI & IND-KI \\
\midrule
Endpoint coupling & Minibatch OT & Independent \\
Coupling construction & Recomputed per update & Resampled per update \\
Endpoint batch size & $128$ & $128$ \\
Mean path & Cubic Hermite & Cubic Hermite \\
Noise schedule & Constant amplitude & Constant amplitude \\
Frequency parameter $\omega$ & $\pi\sqrt{3}$ & $\pi\sqrt{3}$ \\
{Sampling method} & {Score-corrected SDE} & {Score-corrected SDE} \\
{Sampling diffusion $\kappa$ ($\beta=2\kappa$)} & {$0.5$ ($1$)} & {$0.5$ ($1$)} \\
{Integrator / steps per snapshot interval} & {Euler--Maruyama / $125$} & {Euler--Maruyama / $125$} \\
Networks & Acceleration and scaled score & Acceleration and scaled score \\
MLP hidden width & $256$ & $256$ \\
MLP hidden depth & $2$ & $2$ \\
Activation & ReLU & ReLU \\
Training steps per seed & $10{,}000$ & $10{,}000$ \\
Optimizer & Adam & Adam \\
Learning rate & $10^{-3}$ & $10^{-3}$ \\
Score-loss weight & $1$ & $1$ \\
EMA decay & $0.99$ & $0.99$ \\
\bottomrule
\end{tabular}}
\caption{Model and training settings shared across both vortex datasets for GoM.}
\label{tab:oceans_hyperparameters}
\end{table}

\end{document}